\documentclass{article}[12pt]
\usepackage{graphicx}
\usepackage{natbib}
\usepackage{amsmath}
\usepackage{amssymb}
\usepackage{amsthm}
\begin{document}
\title{On the Convergence and Consistency of the Blurring Mean-Shift Process}
\author{Ting-Li Chen\\
              Institute of Statistical Science, \\Academia Sinica, Taipei 11529,
              Taiwan\\
              E-mail: tlchen@stat.sinica.edu.tw}           

\maketitle
\begin{abstract}
The mean-shift algorithm is a popular algorithm in computer vision
and image processing. It can also be cast as a minimum
gamma-divergence estimation. In this paper we focus on the
``blurring'' mean shift algorithm, which is one version of the
mean-shift process that successively blurs the dataset. The analysis
of the blurring mean-shift is relatively more complicated compared
to the nonblurring version, yet the algorithm convergence and the
estimation consistency have not been well studied in the literature.
In this paper we prove both the convergence and the consistency of
the blurring mean-shift. We also perform simulation studies to
compare the efficiency of the blurring and the nonblurring versions
of the mean-shift algorithms. Our results show that the blurring
mean-shift has more efficiency.  

{\bf keywords}
Mean-shift, Convergence, Consistency, Clustering,
$\gamma$-divergence, Super robustness.
\end{abstract}

\renewcommand{\labelenumi}{(\roman{enumi})}
\newtheorem{example}{Example}
\newtheorem{definition}{Definition}
\newtheorem{theorem}{Theorem}
\newtheorem{lemma}{Lemma}
\newtheorem{corollary}{Corollary}
\newtheorem{remark}{Remark}

\section{Introduction}
The mean-shift algorithm is a popular algorithm in computer vision
and image processing. It was initially designed for kernel density
estimation \citep{fukunaga}, which iteratively uses the sample mean
within a local region to estimate the gradient of a density
function. The mean-shift algorithm was further extended and analyzed
by \citet{cheng}. \citet{comaniciu} later applied the mean-shift
algorithm to the problem of image segmentation. Since then the
algorithm has become more well-known in the computer science
community than in the statistics community. For more related works
on the mean-shift algorithm, see \citet{Fashing,Carreira,Carreira2}.
In recently years, methods that use iterative processes on
minimizing $\gamma$-divergence were proposed for robust parameter
estimation \citep{Fujisawa} and for robust clustering \citep{Chen2}.
These methods can also be viewed as the mean-shift based approaches.

Suppose $S=\{x_1,\ldots,x_N\}$ are sample points and
$T=\{y_1,\ldots,y_M\}$ are cluster centers. The nonblurring
mean-shift updating rule can be defined as follows:
\begin{equation}\label{eq:ms}
y_i^{(t+1)}= \sum_{j=1}^N  \displaystyle \frac{f(x_j-y_i^{(t)})
w(x_j)x_j} {\sum_{k=1}^N f(x_k-y_i^{(t)}) w(x_k)},
\end{equation}
where $f$ is a kernel function, $w$ is a weight function, and
$y_i^{(0)}=y_i$. The convergence of the nonblurring version of
mean-shift was studied in \cite{cheng},
\cite{comaniciu2,comaniciu3}, and \cite{li}.

When $T=S$, the updating rule becomes
\begin{equation}\label{eq:bms}
x_i^{(t+1)}= \sum_{j=1}^N  \displaystyle
\frac{f(x_j^{(t)}-x_i^{(t)}) w(x_j)x_j^{(t)}} {\sum_{k=1}^N
f(x_k^{(t)}-x_i^{(t)}) w(x_k^{(t)})},
\end{equation}
where $x_i^{(0)}=x_i$. This is called the blurring mean-shift. Note
that the weighted average is over the updated data points, instead
of the original data. The convergence analysis on the blurring
mean-shift is therefore more complicated than the nonblurring one.
\citet{cheng} proved the convergence of the blurring mean-shift
algorithm for the following two limited cases. When the mutual
influence between each pair of data points is nonzero, Theorem $3$
in \citet{cheng} showed that all data points eventually converge to
a single cluster. When in practice the iterative process is
simulated by a digital computer such that data points can never go
arbitrarily close to each other, Theorem $4$ in \citet{cheng}
guaranteed that the algorithm converges in a finite number of steps.
In Section 2, we show that there is a gap in the proof of Theorem
$4$ by \citet{cheng}. We also discuss related work and the condition
on $f$ and $w$.

In Section 3, we present a more general result on the convergence of
the blurring mean-shift algorithm than Theorem $4$ in \citet{cheng}.
The convergence of the blurring mean-shift is guaranteed under the
general definition: data points eventually become arbitrarily close
to some locations. Since the number of data points is always finite,
there exists a common $t^*$, such that each data point is close
enough to where it converges after the $t^*$-th iteration. That is
to say, the convergence under the general definition can imply the
convergence in a finite number of steps subject to floating point
precision. In addition, Theorem $3$ in \cite{cheng} is an immediate
implication of our result, which is listed in our Corollary
\ref{cor:f>0}.

While the mean-shift algorithm is originally designed for mode
seeking using kernel density estimation, it is questioned that
whether this estimation produces results that converge to the true
parameter values when the number of data points goes to infinity.
\cite{Windham} proposed a robust model fitting, which can be viewed
as a nonblurring approach. \citet{Fujisawa} proposed a robust
estimation by minimizing $\gamma$-divergence and proved the
consistency of their proposed estimation. This is also a nonblurring
approach. In the literature, the consistency of blurring processes
has not been well studied. We present the consistency of the
blurring processes in Section 4.

In additional to convergence and consistency, in Section 5 we
present simulation studies to compare the performance of the
blurring and the nonblurring processes. Discussions and conclusions
are given in Section 6.

\renewcommand{\labelenumi}{(\roman{enumi})}

In this section we present a proof of the convergence of the
blurring mean-shift process. We will first discuss related work, and
introduce some conditions on $f$ and $w$ in (\ref{eq:bms}).

\section{Related Work and Conditions}

Before we start the proof of convergence, it is necessary to bring
out some of our  comments on related works \cite{cheng, Chen}.

\subsection{A GAP in the Proof of Theorem 4 in \cite{cheng}}

As mentioned in the previous section, there is a gap in the proof of
Theorem 4 in \cite{cheng}. Quote from the proof of Theorem 4 in
\citep{cheng}:
\begin{quotation}
Lemma 2 says that the radius of data reaches its final value in
finite number of steps. Lemma 2 also implies that those points at
this final radius will not affect other data points or each other.
Hence, they can be taken out from consideration for further process
of the algorithm.
\end{quotation}
This implication of Cheng's Lemma $2$ is questionable in two
respects. First, when the radius of data points reaches its final
value, it is not trivial to conclude that there do not exist two
data points alternatively switching their locations to be at the
final radius, meaning that data points in such a situation fail to
converge. Although this situation will not happen during the
mean-shift iterative process, it requires to be proven. See our
Lemma \ref{lemma:cv} and its proof.

Second, the convergence of some points at the final radius does not
imply that these points do not affect other points. Although these
points no longer move, it is possible that they still receive
influences from other points, which are just too small to induce a
move larger than the floating point precision. The accumulated
influences from these converged data points at the same location may
be large enough to affect other data points and to induce them a
different move. Therefore, these converged data points should not be
immediately taken out for future process of the algorithm.

\subsection{The weight function $w$}

It was stated \citep{cheng} that the weight function $w$ can be
either fixed through the process or re-evaluated after each
iteration, the convergence was only studied for the case when $w$ is
fixed. In fact, we found that the process does not converge for
arbitrary $w$'s that change over the iterations. The following
example illustrates this.

\begin{example}
\end{example}
Assume the number of data points is 3. Let $x_1=\delta_1$,
$x_2=1/2+\delta_2$, $x_3=-1/2-\delta_3$, where $0<\delta_i<1/4$. Let
\[
f(d)=\left\{
\begin{array}{lll}
1&\quad& d=0,\\
1/2 &\quad& 0<d<1, \\
0 &\quad& 1<d.
\end{array}\right.
\]
Since $x_2-x_3>1$, $f(x_2-x_3)=0$, meaning that $x2$ and $x3$ do not
influence each other in the next update. Let $w(x)=1$ for $-1/2 < x
< 1/2$. Therefore, $w(x_1)=1$. Now we can assign large value to
$w(x_2)$ and $w(x_3)$ so that
\begin{eqnarray*}
x_2^{(1)}&=&\frac {w(x_2)x_2+x_1/2}{w(x_2)+1/2}>1/2,\\
x_3^{(1)}&=&\frac {w(x_3)x_3+x_1/2}{w(x_3)+1/2}<1/2.
\end{eqnarray*}
We can also assign a large enough value to $w(x_3)$, so that
\[
-1/2 < x_1^{(1)}=\frac {x1+w(x_2)x_2/2+w(x_3)x_3/2}{1+w(x_2)+w(x_3)}
<0.
\]
These inequalities show that after the first update, $x_1^{(1)}$
becomes negative, and $x_2^{(1)}$ and $x_3^{(1)}$ remain outside
[-1/2, 1/2].

At each iteration, we can assign large enough values to $w(x_2)$ and
$w(x_3)$, so that $x_1^{(t)}$ is positive when $t$ is even and is
negative when $t$ is odd. We can further control the absolute value
of $x_1^{(t)}$ to be away from zero, so that $x_1^{(t)}$ and
consequently the whole system do not converge. Note that $x_2^{(t)}$
and $x_3^{(t)}$ do converge in this case.

Having seen the above example, in the next section we only prove the
convergence under the condition when $w(x_i^{(t)})$'s are fixed
throughout the process meaning that $w(x_i^{(t)})$'s depend on $i$.
It is worth noted that the convergence of the iterative process in
fact also holds for varying $w(x_i^{(t)})$'s with $\lim_t
w(x_i^{(t)})$ existing for each $i$.

\subsection{The influence function $f$}

While the mean-shift algorithm was originally developed for kernel
density estimation, it is natural to have $f$ in (\ref{eq:bms}) to
be integrable. A weaker condition of $f$, however, suffices to
guarantee the convergence of the iterative process.

\citet{Chen} proposed a self-updating process (SUP) for clustering
as follows:
\begin{enumerate}
\renewcommand{\labelenumi}{(\roman{enumi})}
\item $x_1^{(0)}, \ldots , x_N^{(0)}\in R^p$ are data points to be clustered.
\item At time $t+1$, every point is updated to
\begin{equation} \label{eq:update}
x_i^{(t+1)}= \sum_{j=1}^N  \displaystyle
\frac{f(x_i^{(t)},x_j^{(t)}) } {\sum_{k=1}^N
f(x_i^{(t)},x_k^{(t)})}x_j^{(t)},
\end{equation}
where $f$ is some function that measures the influence between two
data points at time $t$.
\item Repeat (ii) until every point converges.
\end{enumerate}

Although not specified in the notation, the $f$ function in
(\ref{eq:update}) is allowed to be inhomogeneous with respect to
$t$. That is to say, it is more general compared to the $f$ function
in the mean-shift updating rule in (\ref{eq:bms}). \cite{Chen} has
demonstrated the use inhomogeneous $f$'s in several of their
experiments. The $f$ function in (\ref{eq:update}) does not require
to be integrable. It is proposed to satisfy the following PDD
condition.

\begin{definition}
The function $f$ in (\ref{eq:update}) is PDD (positive and
decreasing with respect to distance), if
\begin{enumerate}
\item $0 \leq f(u,v) \leq$ 1, and $f(u,v)=1$ if and only if $u=v$.
\item $f(u,v)$ depends only on $\|u-v\|$, the distance from u to v.
\item $f(u,v)$ is decreasing with respect to $\|u-v\|$,
\end{enumerate}
\end{definition}

Note that  $f$ in (\ref{eq:bms}) is already defined to be only
depending on $u-v$. In the following, we will prove the convergence
under (i) $f$ is PDD and (ii) $w(x_i^{(t)})$ only depends on $i$.

\section{Convergence}

\begin{theorem}\label{thm:main}
If the function $f$ in (\ref{eq:bms}) is PDD, and if the weight
function $w(x_j^{(t)})=w_j$ in (\ref{eq:bms}) depends only on $j$,
there exists $\{x_1^*, \ldots, x_N^*\}$, such that
\[
\lim_{t \to \infty} x_i^{(t)} = x_i^* \quad\quad \forall i.
\]
\end{theorem}

Below we outline the proof for Theorem \ref{thm:main}.
\begin{itemize}
\item First, consider the convex hull of all data points in each iteration.
The convex hulls with respect to iterations are nested (Lemma
\ref{lemma:monotone}) and converge.
\item Next, for each vertex of the converged convex hull, there exists at least one sequence of the updated data points converging to this vertex (Lemma \ref{lemma:cv}).
\item The influence from the converged data points at the vertices of the converged convex hull goes down to zero to other data points (Lemma \ref{lemma:fto0}).
\item Consider the convex hull of the rest data points (exclude those already converged).
Using the same arguments again, we have a few more converged data
points. We can repeat this process over and over again until all
data points converge.
\end{itemize}

\begin{definition}
The convex hull $C(X)$ for a set of points $X$ in a vector space ${\cal V}$ is the minimal convex set containing $X$.\\
\end{definition}

\begin{lemma}\label{lemma:monotone}
Let $C_1^{(t)}$ be the convex hull of $\{x_1^{(t)}, \ldots,
x_N^{(t)}\}$. Then
\[
C_1^{(0)} \supseteq \ldots \supseteq C_1^{(t)} \supseteq \ldots.
\]
\end{lemma}
\begin{proof}
The convex hull $C(X)$ for a set of points $X$ is the minimal convex
set containing $X$. Since
\[
x_i^{(t+1)}= \frac {\displaystyle{\sum_{j=1}^N}
f(x_i^{(t)}-x_j^{(t)}) w_j x_j^{(t)}} {\displaystyle{\sum_{j=1}^N}
f(x_i^{(t)}-x_j^{(t)})w_j},
\]
$x_i^{(t+1)}$ is a weighted average of $x_j^{(t)}$ for $j=1, \ldots,
N$. Therefore, $ x_i^{(t+1)} \in C_1^{(t)}$. Since the above is true
for each $i$, we have
\[
C_1^{(t)} \supseteq C(\{x_1^{(t+1)}, \ldots,
x_N^{(t+1)}\})=C_1^{(t+1)}.
\]
\end{proof}

Note that the nested structure presented in Lemma
$\ref{lemma:monotone}$ ensures the convergence of convex hulls
$\{C_1^{(t)}\}$. Let $C_1$ be the limit of $C_1^{(t)}$,
\[
C_1 \equiv \lim_{t \to \infty} C_1^{(t)}= \bigcap_{t=0}^{\infty}
C_1^{(t)}.
\]
On the other hand, since the convex hull of any finite set of points
in $R^p$ is a polytope, each $C_1^{(t)}$ is a polytope. Each vertex
of $C_1^{(t)}$ therefore must contain at least one $x_i^{(t)}$ for
some $i$, otherwise the polytope would have been smaller. With the
convergence of convex hulls $\{C_1^{(t)}\}$, Lemma $\ref{lemma:cv}$
claims that for each vertex of $C_1$, there exists at least one
equence of $\{x_i^{(t)}\}$ which converges to this vertices.

\begin{lemma}\label{lemma:cv}
If the function $f$ in (\ref{eq:bms}) is PDD, for each vertex
$v_{1,i}$ of $C_1$, there exists at least one $j$, such that
\begin{equation}\label{eq:conv_1}
\lim_{t \to \infty} x_j^{(t)} = v_{1,i}.
\end{equation}
\end{lemma}
\begin{proof}
Since $C_1 = \lim_{t \to \infty} C_1^{(t)}$ for each $i$, there
exists a sequence of $v_{1,i}^{(t)}$'s (exchange vertex indices if
necessary), such that $\lim_{t \to \infty} v_{1,i}^{(t)}=v_{1,i}$,
where $v_{1,i}^{(t)}$ is a vertex of $C_1^{(t)}$. Since for any $t$
and $i$, $v_{1,i}^{(t)}=x_k^{(t)}$ for at least one $k$, there
exists $j$, such that $x_j^{(t)}=v_{1,i}^{(t)}$ for infinite many
$t$'s. Therefore, there exists an infinite time sequence $t_n$'s,
such that
\[
x_j^{(t_n)}=v_{1,i}^{(t_n)}\quad\quad \forall n,
\]
which leads to
\[
\lim_{n \to \infty} x_j^{(t_n)} = v_{1,i}.
\]
If $x_j^{(t)}=v_{1,i}^{(t)}$ except for any finite $t$, then
equation (\ref{eq:conv_1}) is established. Otherwise, there exists
$j' \neq j$ and another infinite time sequence $s_n$'s, such that
\[
x_{j'}^{(s_n)}=v_{1,i}^{(s_n)} \quad\quad \forall n.
\]
Without loss of generality, assume that $v_{1,i}^{(t)}=x_j^{(t)}$ or
$x_{j'}^{(t)}$ for all $t > \tilde{t}$. Assume $w_j\geq w_{j'}$.
From equation (\ref{eq:update}), if $x_j^{(s)}=x_{j'}^{(s)}$ for
some $s$, $x_j^{(t)}=x_{j'}^{(t)}$ for all $t>s$. Therefore, for any
$s>0$, there exists $t>s$, such that $v_{1,i}^{(t)}=x_j^{(t)}$ and
$v_{1,i}^{(t+1)}=x_{j'}^{(t+1)}$.  We claim that this case, however,
can never happen: when $t$ is large enough, it is impossible that a
data point inside the convex hull later becomes a new vertex, since
it is closer to other points than the current vertex is. In the
following we prove this claim only for the one dimensional case. For
higher dimensional cases, consider the supporting hyperplane
contained $v_{1,i}$. Since $v_{1,i}$ is a vertex of a convex set, a
supporting hyperplane can be chosen such that no other point is in
the hyperplane. Now we can project all data points onto to the
straight line which is perpendicular to the supporting hyperplane
and pass through $v_{1,i}$. Then we can make the same argument on
the projected data points.

Without loss of generality, assume $v_{1,i}=0$, $x_j^{(t)}\leq 0$,
and $x_k^{(t)}>0$ for $k \ne j$ or $j'$. If $x_{j'}^{(t+1)}$ later
becomes the new vertex, then
\begin{equation}\label{eq:new_v}
\frac {\displaystyle{\sum_{k=1}^N} f(x_{j'}^{(t)}-x_k^{(t)})w_k
x_k^{(t)} } {\displaystyle{\sum_{k=1}^N}
f(x_{j'}^{(t)}-x_k^{(t)})w_k} < \frac {\displaystyle{\sum_{k=1}^N}
f(x_{j}^{(t)}-x_k^{(t)})w_k x_{k}^{(t)}}
{\displaystyle{\sum_{k=1}^N} f(x_{j}^{(t)}-x_k^{(t)})w_k}.
\end{equation}
Moreover, since $x_{j'}^{(t+1)}$ is the new vertex,
\[
\frac {\displaystyle{\sum_{k=1}^N} f(x_{j'}^{(t)}-x_k^{(t)}) w_k
x_k^{(t)} } {\displaystyle{\sum_{k=1}^N}
f(x_{j'}^{(t)}-x_k^{(t)})w_k} \leq 0 \quad \Longrightarrow \quad
\displaystyle{\sum_{k=1}^N} f(x_{j'}^{(t)}-x_k^{(t)}) w_k x_k^{(t)}
\leq 0.
\]
Since $x_j^{(t)}$ is the current vertex, $\| x_j^{(t)}-x_k^{(t)} \|>
\| x_{j'}^{(t)}-x_k^{(t)} \|$ for all $k$, and hence
$f(x_{j}^{(t)}-x_k^{(t)}) < f(x_{j'}^{(t)}-x_k^{(t)})$. Then
\begin{eqnarray*}
&&\displaystyle{\sum_{k=1}^N} f(x_{j'}^{(t)}-x_k^{(t)})w_kx^{(t)}_k \\
&=&\displaystyle{w_{j'}x^{(t)}_{j'}+f(x_{j'}^{(t)}-x_j^{(t)})w_j
x^{(t)}_j+\sum_{k\ne j,j'}} f(x_{j'}^{(t)}-x_k^{(t)}) w_k
x^{(t)}_k \\
&\geq&\displaystyle{w_j x^{(t)}_{j}+f(x_{j'}^{(t)}-x_j^{(t)}) w_{j'}
x^{(t)}_{j'}+\sum_{k\ne j,j'}} f(x_{j}^{(t)}-x_k^{(t)}) w_k
x^{(t)}_k\\
&=&\displaystyle{\sum_{k=1}^N} f(x_{j}^{(t)}-x_k^{(t)})w_k x^{(t)}_k
.
\end{eqnarray*}
Since
\[
\sum_{k=1}^N f(x_{j}^{(t)}-x_k^{(t)})w_k x^{(t)}_k \leq \sum_{k=1}^N
f(x_{j'}^{(t)}-x_k^{(t)})w_k x^{(t)}_k <0,
\]
and
\[
0<\sum_{k=1}^N f(x_{j}^{(t)}x_k^{(t)})w_k < \sum_{k=1}^N
f(x_{j'}^{(t)}x_k^{(t)})w_k,
\]
we have
\[
\frac {\displaystyle{\sum_{k=1}^N} f(x_{j'}^{(t)}-x_k^{(t)}) w_k
x_k^{(t)} } {\displaystyle{\sum_{k=1}^N}
f(x_{j'}^{(t)}-x_k^{(t)})w_k} < \frac {\displaystyle{\sum_{k=1}^N}
f(x_{j}^{(t)}-x_k^{(t)}) w_k x_{k}^{(t)}}
{\displaystyle{\sum_{k=1}^N} f(x_{j}^{(t)}-x_k^{(t)})w_k},
\]
which is a contradiction to (\ref{eq:new_v}).
\end{proof}

Having shown that at least some points converge under the iterative
updates, hereafter we consider the rest of the data points. Let
$\Omega_1$ be the set of points shown converging to the vertices of
$C_1$. Define $C_2^{(t)}$ be the convex hull of $\{x_i^{(t)}\}_{i
\notin \Omega_1}$. Note that $\{C_2^{(t)}\}$ may not be nested at
early stages of iterations: points not in $\Omega_1$ may move
outside the current convex hull $C_2^{(t)}$ due to the influence
from $\Omega_1$, the volume of the convex hull therefore may
increase by iteration. This nested property, however, would hold
after some iteration when all data points in $\Omega_1$ converge.
Explicitly,
\[
C_2^{(t)} \supseteq C_2^{(t+1)} \quad \forall t \geq \tilde{t}
\mbox{ for some } \tilde{t},
\]
which also implies the convergence of $\{C_2^{(t)}\}$,
\[
C_2 \equiv \lim_{t \to \infty} C_2^{(t)}.
\]
We introduce the following Lemma $\ref{lemma:fto0}$, which can lead
to the nested property of $\{C_2^{(t)}\}$. It states that when all
data points in $\Omega_1$ converge, points in $\Omega_1$ receive no
influence from points not in $\Omega_1$, otherwise they would have
been attracted inwards. That is to say, data points not in
$\Omega_1$ also no longer receive influence from points in
$\Omega_1$, meaning that the influence from points in $\Omega_1$
goes down to zero.

\begin{lemma}\label{lemma:fto0} For an arbitrary $x_i\in \Omega_1$,
we have
\[
\lim_{t \to \infty} f(x_i^{(t)}-x_j^{(t)})=0,
\]
for all $j$ such that $\lim_{t \to \infty} x_j^{(t)} \ne \lim_{t \to
\infty} x_i^{(t)}$.
\end{lemma}
\begin{proof}
Without loss of generality, assume that $x_i^{(t)}$ is the only data
point that converges to $v_{i,1}$.
\begin{eqnarray}
&\frac {\displaystyle{\sum_{j=1}^N} f(x_{i}^{(t)}-x_j^{(t)}) w_j
x_j^{(t)}} {\displaystyle{\sum_{j=1}^N}
f(x_{i}^{(t)}-x_j^{(t)})w_j}=x_i^{(t+1)}&\nonumber\\\nonumber
\Rightarrow&\frac {\displaystyle{\sum_{j=1}^N}
f(x_{i}^{(t)}-x_j^{(t)}) w_j \cdot (x_j^{(t)}-x_i^{(t+1)})}
{\displaystyle{\sum_{j=1}^N} f(x_{i}^{(t)}-x_j^{(t)})w_j}=0&\\
\Rightarrow& \displaystyle{\sum_{j \ne i}^N}
f(x_{i}^{(t)}-x_j^{(t)}) w_j \cdot (x_j^{(t)}-x_i^{(t+1)})=w_i \cdot
(x_i^{(t+1)}-x_i^{(t)}).& \label{eq:eq5}
\end{eqnarray}

Since $x_i^{(t)}$ converges to $v_{i,1}$, $x_i^{(t+1)}$ and
$x_i^{(t)}$ become arbitrarily close to each other when $t$ is large
enough. That is, the right-hand side of (\ref{eq:eq5}) goes down to
zero. On the other hand, since $x_j^{(t)}$ does not converge to
$v_{i,1}$ for $j \neq i$, there is a gap between $x_j^{(t)}$ and
$x_i^{(t+1)}$. To force the left-hand side of (\ref{eq:eq5}) to be
zero, $f(x_{i}^{(t)}-x_j^{(t)})$ must go down to zero as well. This
sketches the proof for Lemma \ref{lemma:fto0}. The precise details
are given in the following.

Because $x_j^{(t)}$ does not converge to $v_{i,1}$ for $j\ne i$,
there exists $\epsilon>0$, for any $t_0>0$, there exists $t>t_0$
such that $\|x_j^{(t)}-v_{i,1}\|>\epsilon$. In fact, $x_j^{(t)}$ can
not go arbitrarily close to $v_{i,1}$ when $t$ is large enough,
otherwise the updating process will move $x_j^{(t)}$ and $x_i^{(t)}$
closer and closer to each other. That is, there exists
$\epsilon_0>0$ and $t_1$ such that
$\|x_j^{(t)}-v_{i,1}\|>\epsilon_1$ for all $t>t_1$. On the other
hand, because $x_i^{(t)} \to v_{i,1}$, for any $\epsilon_2>0$, there
exists  $t_2$, such that $\|x_i^{(t)}-x_i^{(t+1)}\|<\epsilon_2$ for
$t>t_2$.

Since $v_{1,i}$ is a vertex of the convex set $C_1$, there exists
$x\in C_1$, such that the inner product of $x-v_{1,i}$ and
$y-v_{1,i}$ is positive for any $y\in C_1$. Let
\[
v_x=\frac {x-v_{1,i}} {\|x-v_{1,i}\|}.
\]
There exists $\alpha>0$ and $t_3>t_1$ such that
\[
\langle x_j^{(t)} -  v_{1,i}, v_x\rangle \geq \alpha \|x_j^{(t)} -
v_{1,i}\| \quad \quad \forall t>t_3 \mbox{ and } \forall j \ne i,
\]
where $\langle, \rangle$ denotes the inner product. Take the inner
product of both sides of (\ref{eq:eq5}) with $v_x$, we have
\begin{eqnarray*}
&&\left<\sum_{j \ne i}^N f(x_{i}^{(t)}-x_j^{(t)}) w_j \cdot
(x_j^{(t)}-x_i^{(t+1)}) , v_x \right> \\
&=&\sum_{j \ne i}^N f(x_{i}^{(t)}-x_j^{(t)}) w_j \cdot
\left<x_j^{(t)}-x_i^{(t+1)} , v_x \right> \\
&=&\sum_{j \ne i}^N f(x_{i}^{(t)}-x_j^{(t)})w_j \cdot  \left(
\left<x_j^{(t)}-v_{1,i} , v_x \right> +\left<v_{1,i}-x_i^{(t+1)} , v_x \right> \right)\\
&\geq& \sum_{j \ne i}^N f(x_{i}^{(t)}-x_j^{(t)})  w_j\alpha
\|x_j^{(t)}-v_{1,i}\| \\
&>& \max_j w_j \cdot \alpha \epsilon_1 \sum_{j \ne i}^N f(x_{i}^{(t)}-x_j^{(t)})  \\
\end{eqnarray*}
for $t>t_3$, and
\[
\left<x_i^{(t+1)}-x_i^{(t)} , v_x \right> \leq \|
x_i^{(t+1)}-x_i^{(t)}\| < \epsilon_2
\]
for $t>t_2$. Therefore, for $t> \max (t_3,t_2)$,
\[
\max_j w_j \cdot\alpha \epsilon_1 \sum_{j \ne i}^N
f(x_{i}^{(t)}-x_j^{(t)}) < w_i\epsilon_2.
\]
Since $\epsilon_2$ can be arbitrarily small, the inequality above
implies
\[
\sum_{j \ne i}^N f(x_{i}^{(t)}-x_j^{(t)}) \to 0.
\]
Since $f\geq 0$, $f(x_{i}^{(t)}-x_j^{(t)}) \to 0$ for all $j\ne
i$.
\end{proof}

From the above, we can claim a similar result for $C_2$ as Lemma
$\ref{lemma:cv}$ for $C_1$: each of the vertex of $C_2$ has at least
one data point converges to. The same argument can apply again and
again to $C_3$, $C_4$, $\ldots$, until all data points converge.
This completes the proof of Theorem \ref{thm:main}.

Although Theorem \ref{thm:main} guarantees the convergence when $f$
has {\it PDD} condition, there are some $f$'s that produce trivial
clustering results, in which all data points are clustered into one
single group. We identify such $f$'s in the following corollary.
\begin{corollary}\label{cor:f>0}
Let $r_M\equiv \max_{i,j}\{||x_i-x_j||\}$. If $f$ is PDD with
$f(r_M)>0$, then there exists $c$, such that
\[
\lim_{t \to \infty}x_i^{(t)}=c \quad\quad \forall i.
\]
\end{corollary}
\begin{proof}
Lemma \ref{lemma:monotone} implies that $||x_i^{(t)}-x_j^{(t)}||
\leq r_M$ for every $t$, $i$ and$j$. Since $f$ is decreasing with
respect to distance, $f(x_i^{(t)}- x_j^{(t)}) \geq f(r_M)>0$. Lemma
\ref{lemma:fto0} shows that, however, the influence between any two
points which do not converge to the same position tends to zero.
Thus, $f(x_i^{(t)}- x_j^{(t)}) \geq f(r_M)>0$ for every $i$ and $j$,
which implies that all data points converge to the same position.
\end{proof}

For the purpose of clustering, it is not desirable to have all data
points converged to the same position. To prevent trivial clustering
results, $f$ has to be zero on $(r, \infty)$ for some $r<r_M$.

\section{Consistency}

In the previous section, we proved the convergence of the algorithm.
In this section, we study the estimation consistency of the
algorithm. We show the consistency for the Normal case and remark on
more general cases. The difficulty of our consistency proof arises
from blurring process, i.e., the the iterative data shrinkage
update.

Assume $x_i$'s $\in R^p$ are i.i.d. sampled from $N(0,\Sigma)$, and
the mutual influence function $f$ adopted is $\exp (-(x-y)^\top
(x-y)/2\tau^2)$, where $(x-y)^\top$ is the transpose of vector
$x-y$. Assume $w=1$. The updating rule is:
\begin{equation} \label{eq:updatem}
x_{i,n}^{(t+1)}= \sum_{j=1}^N  \displaystyle
\frac{f(x_{i,n}^{(t)}-x_{j,n}^{(t)}) } {\sum_{j=1}^N
f(x_{i,n}^{(t)}-x_{j,n}^{(t)})}x_{j,n}^{(t)},
\end{equation}
where $x_{i,n}^{(t)}$ denotes the updated $x_i$ at $t$-th iteration
when considering only first $n$ samples. By Corollary \ref{cor:f>0}
presented in the previous section, we know that for all $i$
\[
\lim_{t \to \infty} x_{i,n}^{(t)}=c
\]
for the same $c$. Here we want to show that $c$ will converge to
zero almost surely, which we state as the following theorem:
\begin{theorem}\label{thm:consit}
\[
\lim_{n\to \infty}\lim_{t \to \infty} x_{i,n}^{(t)}=0 \quad
\mbox{a.s.}
\]
\end{theorem}
\begin{proof}

Let $G(x;\Sigma)$ be the CDF of $N(0,\Sigma)$, $G_n^{(t)}(x)$ be the
empirical CDF of the $n$-sample at $t$-th iteration, and
$G^{(t)}(x)=\lim_{n \to \infty}G_n^{(t)}(x)$. By Glivenko-Cantelli
theorem,
\[
\lim_{n \to \infty} \sup_x |G_n^{(0)}(x) - G(x,\Sigma)|=0 \quad
{\mbox a.s.}
\]
We claim that the the empirical distribution of the updated data
points of each iteration converges to a Normal distribution. In the
following, we show that
\begin{equation}\label{eq:conv}
\lim_{n \to \infty} \sup_x |G_n^{(t)}(x) - G^{(t)}(x)|=0 \quad
\mbox{a.s.}
\end{equation}
where $G^{(t)}(x)= G(x ;\Sigma_t)$.
This is true for $t=0$. Assume that it is true for $t=s$, we want to
show that it is true for $t=s+1$. Assume that
\[
\sup_x |G_n^{(s)}(x) - G^{(s)}(x)|<\epsilon_s,
\]
for $n>N_{\epsilon_s}$. Define
\[
K_H(x)=\frac {\int_y f(x-y)\cdot  y \cdot dH(y)} {\int_y f(x-y)\cdot
dH(y)}.
\]
With the assumption that $G^{(s)}(x)=G( x; \Sigma_s)$, we have
\begin{eqnarray*}
f(x-y)dG^{(s)}
&=& c_s \exp({-\frac {(x-y)^\top(x-y)}{2 \tau^2}}) \cdot \exp({-\frac {y^\top \Sigma_s^{-1} y}{2 }})dy\\
&=&c_s \exp\left[-\frac 12 \left\{ \frac 1{\tau^2}( {x^\top x}-{2x^\top y}) + y^\top(I/\tau^2+\Sigma_s^{-1})y\right\}\right]dy\\
&=&c_s'(x) \exp\left[-\frac 12 \left\{
y-(I+\tau^2\Sigma_s^{-1})^{-1}x\right\}^\top
(I/\tau^2+\Sigma_s^{-1})\right.\\
&&\left.\vphantom{\frac
12}\left\{y-(I+\tau^2\Sigma_s^{-1})^{-1}x)\right\}\right]dy.
\end{eqnarray*}
Therefore,
\begin{equation}\label{eq:shrink}
K_{G^{(s)}}(x)=(I+\tau^2\Sigma_s^{-1})^{-1}x.
\end{equation}
Since
\[
|G_n^{(s)}(x)-G^{(s)}(x)|<\epsilon_s
\]
and $f(x-y)y$ and $f(x-y)$ are bounded, we have
\begin{equation}\label{eq:conv_K}
||K_{G_n^{(s)}}(x) - K_{G^{(s)}}(x)||_2 < \alpha_s \epsilon_s
\end{equation}
for some positive number $\alpha_s$ where $||\cdot||_2$ is the $L^2$
norm. Since
\begin{eqnarray*}
K_{G_n^{(s)}}(x_{i,n}^{(s)})&=&\frac {\int_y f(x_{i,n}^{(s)}-y)\cdot  y \cdot d{G_n^{(s)}(y)}} {\int_y f(x_{i,n}^{(s)}-y)\cdot dG_n^{(s)}(y)}\\
&=&\frac{\sum_{j=1}^N f_s(x_{i,n}^{(s)}-x_{j,n}^{(s)})x_{j,n}^{(s)}}
{\sum_{j=1}^N
f_s(x_{i,n}^{(s)}-x_{j,n}^{(s)})}\\
&=&x_{i,n}^{(s+1)},
\end{eqnarray*}
we have
\begin{eqnarray*}
&& ||x_{i,n}^{(s+1)} - (I+\tau^2\Sigma_s^{-1})^{-1}x_{i,n}^{(s)}||_2\\
&=& ||K_{G_n^{(s)}}(x_{i,n}^{(s)}) - K_{G^{(s)}}(x_{i,n}^{(s)})||_2\\
&<&\alpha_s \epsilon_s.
\end{eqnarray*}
The empirical distribution of $x_{i,n}^{(s+1)}$ is $G_n^{(s+1)}(x)$,
and that of $(I+\tau^2\Sigma_s^{-1})^{-1}x_{i,n}^{(s)}$ is
$G_n^{(s)} ((I+\tau^2\Sigma_s^{-1})x)$. Then
\begin{eqnarray*}
&& \left|G_n^{(s+1)}(x)-G^{(s)} \left((I+\tau^2\Sigma_s^{-1})x\right)\right|\\
&\leq& \max_{||\Delta x ||<\alpha_s \epsilon_s } \left|G_n^{(s)} \left( (I+\tau^2\Sigma_s^{-1})(x+\Delta x)\right)-G^{(s)} \left((I+\tau^2\Sigma_s^{-1})x\right)\right| \\
&\leq& \max_{||\Delta x ||<\alpha_s \epsilon_s } \left\{
\left|G_n^{(s)} \left( (I+\tau^2\Sigma_s^{-1})(x+\Delta
x)\right)-G^{(s)}
\left((I+\tau^2\Sigma_s^{-1})(x+\Delta x)\right)\right|\right.\\
&&+\left. \left|G^{(s)}
\left((I+\tau^2\Sigma_s^{-1})(x+\Delta x)\right)-G^{(s)} \left((I+\tau^2\Sigma_s^{-1})x\right)\right|\right\}\\
&<& \epsilon_s + \max_{||\Delta x ||<\alpha_s \epsilon_s }
\left|G^{(s)}
\left((I+\tau^2\Sigma_s^{-1})(x+\Delta x)\right)-G^{(s)} \left((I+\tau^2\Sigma_s^{-1})x\right)\right|\\
&\leq& \epsilon_s+\max_{||\Delta x ||<\alpha_s \epsilon_s }
||(I+\tau^2\Sigma_s^{-1})\Delta x||_2 \max_{||\Delta x ||<\alpha_s
\epsilon_s }
||\frac {\partial }{\partial x}G^{(s)}(x+\Delta x)||_2\\
&\leq&\epsilon_s +\lambda\alpha_s \epsilon_s \frac
1{\sqrt{2\pi}|\det(I+\tau^2\Sigma_s^{-1})|^{1/2}},
\end{eqnarray*}
where $\lambda$ is the largest eigenvalue of
$I+\tau^2\Sigma_s^{-1}$. Therfore, $|G_n^{(s+1)}(x)-G^{(s)}
((I+\tau^2\Sigma_s^{-1})x)|$ can be arbitrarily small by choosing a
small enough $\epsilon_s$. This completes the induction.

From (\ref{eq:shrink}), we have
\[
\Sigma_{s+1}=(I+\tau^2\Sigma_s^{-1})^{-1} \Sigma_s
(I+\tau^2\Sigma_s^{-1})^{-1}.
\]
Since $\Sigma_s$ is a covariance matrix, it is symmetric and
positive definite. Then $\Sigma_s$ can be factorized as
\[
\Sigma_s=P \Lambda_s P^\top
\]
where $P P^\top=I$ and $\Lambda_s$ is a diagonal matrix. Then
\begin{eqnarray*}
\Sigma_s^{-1}&=&P \Lambda_s^{-1} P^\top,\\
I+\tau^2\Sigma_s^{-1}&=&P (I+\tau^2\Lambda_s^{-1}) P^\top,\\
\Sigma_{s+1}&=&(I+\tau^2\Sigma_s^{-1})^{-1} \Sigma_s
(I+\tau^2\Sigma_s^{-1})^{-1}\\
&=&P (I+\tau^2\Lambda_s^{-1})^{-1} \Lambda_s
(I+\tau^2\Lambda_s^{-1})^{-1} P^\top.
\end{eqnarray*}
Therefore, $\Sigma_s$ and $\Sigma_{s+1}$ share the same
eigenvectors. Assume that $\lambda_i^{(s)}$'s are the eigenvalues of
$\Sigma_s$ and $\lambda_i^{(s+1)}$'s are those of $\Sigma_{s+1}$.
Then
\begin{eqnarray*}
\lambda_i^{(s+1)}&=& (1+\tau^2/\lambda_i^{(s)})^{-1} \lambda_i^{(s)}
(1+\tau^2/\lambda_i^{(s)})^{-1}\\
&=& \frac {(\lambda_i^{(s)})^2}{(\lambda_i^{(s)}+\tau^2)^2} \lambda_i^{(s)}\\
&\leq& \frac {(\lambda_i^{(0)})^2}{(\lambda_i^{(0)}+\tau^2)^2}
\lambda_i^{(s)} \\
&\leq& \left\{\frac
{(\lambda_i^{(0)})^2}{(\lambda_i^{(0)}+\tau^2)^2}\right\}^{s+1}
\lambda_i^{(0)}.
\end{eqnarray*}
Therefore $\lambda_i^{(s)} \to 0$ as $s \to \infty$. For any
$\epsilon$, there exists $t_0$ such that $\min_i \lambda_i^{(t_0)} <
\epsilon^2/k$, where $k$ is a large integer. From (\ref{eq:conv}),
almost surely
\[
\sup_x |G_{n}^{(t_0)}(x) - G^{(t_0)}(x)| \to 0.
\]
Equivalently,
\[
\sup_A |G_{n}^{(t_0)}(A) - G^{(t_0)}(A)| \to 0,
\]
where $G_{n}^{(t_0)}(A)$ and $G^{(t_0)}(A)$ denote the probabilities
of $x \in A$. Therefore, for any $\delta>0$, there exists $n_{t_0}$
such that
\[
\sup_A |G_{n}^{(t_0)}(A) - G^{(t_0)}(A)| <\delta
\]
for all $n>n_{t_0}$. Then
\begin{eqnarray*}
\Pr(||x_{i,n}^{(t_0)}||_2 > \epsilon)&=&G_{n}^{(t_0)}(x^\top x >
\epsilon^2)\\
&<&G^{(t_0)}(x^\top x >
\epsilon^2) +\delta\\
&=&G^{(t_0)}(\frac 1 {\min_i \lambda_i^{(t_0)}} x^\top x > \frac {\epsilon^2} {\min_i \lambda_i^{(t_0)}} )+\delta\\
&\leq &G^{(t_0)}(x^\top
\Sigma_{t_0}^{-1} x > \frac {\epsilon^2} {\min_i \lambda_i^{(t_0)}} )+\delta\\
&\leq & G^{(t_0)}(x^\top \Sigma_{t_0}^{-1} x > k )+\delta\\
&=& G(x^\top x>k ; I_p) +\delta,
\end{eqnarray*}
where $I_p$ is the identity matrix. This can be arbitrarily small by
choosing $k$ large enough and $\delta$ small enough. Therefore,
almost all updated data points are in $B(0,\epsilon)$ at $t_0$-th
iteration, where $B(0,\epsilon)=\{x:||x||_2 > \epsilon\}$. For
iteration $t>t_0$, all updated data points within $B(0,\epsilon)$
will not move outside $B(0,\epsilon)$, since there are more updated
data points and hence more influence in the direction toward to
zero. Therefore, $|x_{i,n}^{(t)}| \leq \epsilon$ for almost all $i$
and for all $t>t_0$ and $n>n_{t_0}$. By Corollary \ref{cor:f>0}, all
data points will converge to a single location. We have
\[
||\lim_{t\to \infty} x_{i,n}^{(t)}||_2 \leq \epsilon.
\]
for all $i$ when $n>n_{t_0}$, which completes the proof.
\end{proof}

\begin{remark}
In this section, we present the results under the assumption that
both $f$ and $G$ are Normal. The results can be generalized to
general second order kernel functions with translation invariance.
For this type of kernel functions, the empirical distribution at
each iteration still converges to some distribution, and the
variance is decreasing through iterations. The shrunk distribution,
however, may not have a nice form as that in the Normal case.
\end{remark}

\begin{remark}
If the data points are sampled from a finite mixture distribution,
the locations which the data points converge to through the
iterative process may not be consistent to the parameters. Take the
mixture distribution  $\alpha_1 N(\mu_1,1)+ (1-\alpha_1) N(\mu_2,1)$
as an example. By choosing a proper $f$, data points will be
clustered into two groups. Since the domains of these two Normal
distribution are overlapped, the converged locations through the
iterative process will not converge to $\mu_1$ and $\mu_2$.
\end{remark}

\section{Simulation}
In this section we consider a one dimensional case where the data is
sampled from $N(0,\sigma_0^2)$. The $f$ function in (\ref{eq:bms})
is taken to be $f=\exp (-(x-y)^2/2\tau^2)$. We used three
experiments to compare the blurring and the nonblurring processes in
the following three aspects: the convergence rate, the efficiency,
and the robustness to the outliers.

\subsection{Convergence rate}

Based on (\ref{eq:shrink}), we have shown that
\[
K_{G^{(s)}}(x)=\frac {\int_y f(x-y)\cdot  y \cdot dG^{(s)}(y)}
{\int_y f(x-y)\cdot
dG^{(s)}(y)}=\frac{\sigma_s^2}{\sigma_s^2+\tau^2}x.
\]
For the nonblurring process, the integration is over the original
data, instead of updated data. The shrinkage ratio is therefore
$\frac{\sigma_0^2}{\sigma_0^2+\tau^2}$, meaning that the convergence
rate of the blurring process is higher than that of the nonblurring
process. Take $\sigma_0=1$ and $\tau=2$ as an example. For the
blurring process,
\begin{eqnarray*}
\sigma_1&=&\sigma_0\frac{\sigma_0^2}{\sigma_0^2+\tau^2}=\frac{1^2}{1^2+2^2}=0.2\\
\sigma_2&=&\sigma_1\frac{\sigma_1^2}{\sigma_1^2+\tau^2}=0.2\frac{0.2^2}{0.2^2+2^2}\approx
0.002\\
\sigma_3&=&\sigma_2\frac{\sigma_2^2}{\sigma_2^2+\tau^2}=0.002\frac{0.002^2}{0.002^2+2^2}\approx
0.000000002.
\end{eqnarray*}
For the nonblurring process,
\begin{eqnarray*}
\sigma^{'}_1&=&\sigma^{'}_0\frac{\sigma_0^2}{\sigma_0^2+\tau^2}=\frac{1^2}{1^2+2^2}=0.2\\
\sigma^{'}_2&=&\sigma^{'}_1\frac{\sigma_0^2}{\sigma_0^2+\tau^2}=0.2\frac{1^2}{1^2+2^2}=
0.04\\
\sigma^{'}_3&=&\sigma^{'}_2\frac{\sigma_0^2}{\sigma_0^2+\tau^2}=0.04\frac{1^2}{1^2+2^2}=
0.008.
\end{eqnarray*}
\begin{figure}[htb]
\centering
\includegraphics[width=\textwidth]{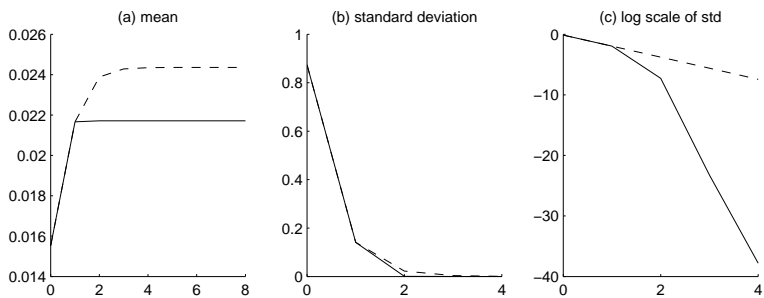}
\caption{The simulation results on 100 samplings from N(0,1). The
solid line is from the blurring process, and the dash line is from
the nonblurring process. } \label{fig:fig1}
\end{figure}

In this experiment, we sampled 100 data points from $N(0,1)$. Fig.
\ref{fig:fig1} presents the simulation results by the blurring and
the nonblurring process. In details, Fig. \ref{fig:fig1}(a) shows
that both processes converged to very close to the true mean of
zero. Fig. \ref{fig:fig1}(b) shows that the standard deviations of
the updated data points dropped way down at the first iteration and
became nearly zero after the second iteration. This illustrates that
both processes converged very fast, while the updated data points by
the blurring process shrunk even much faster. Fig. \ref{fig:fig1}(c)
further presents the shrinkage of the updated data points in terms
of the log scale of the standard deviations in Fig.
\ref{fig:fig1}(b).

\subsection{Efficiency}
In this experiment we consider $\tau$ to be 0.5, 1 or 2. For each
$\tau$ value, we simulated 100,000 sets of 100 data points, which
were again sampled from $N(0,1)$. According to the simulated 100,000
sets, we summarized the means and the standard deviations of the
following three statistics: the sample mean, the number each set of
data points converged to by the blurring process and that by the
nonblurring processes. The results were presented in Table
\ref{tab:nonoise}.

 In this experiment, we consider 100 data points
were sampled from $N(0,1)$. Now we experiments with $\tau=0.5$, 1
and 2. For each parameter, we simulate 100,000 times. The means and
the standard deviations of the sample mean and the converged numbers
of blurring and nonblurring processes in these 100,1000 simulations
are presented in Table \ref{tab:nonoise}. There is no noticeable
difference between the means of three statistics. We did run
multiple 100,000-sample sets, and the orders (with respect to the
absolute value) are different for different sets. However, the
standard deviations of the three statistics are clearly different.
The standard deviations of the sample means are close to 0.1, which
is the theoretic value. The standard deviations of the converged
number from the blurring process are smaller than that from the
nonblurring process. Therefore, the converged number from the
blurring one seems to be a better estimator over that from the
nonblurring one.

\begin{table}[!h]
\caption{The mean and the standard deviation of the converged
points}\label{tab:nonoise} \centering
\begin{tabular}{|l|c|c|c|}\hline
$\tau$ & Sample Mean & Blurring & Nonblurring\\ \hline
0.5&-1.897*10$^{-4}$ (0.1000)&-5.697*10$^{-4}$
(0.1210)&-5.349*10$^{-4}$ (0.2126)\\\hline 1&1.260*10$^{-4}$
(0.0997)&2.400*10$^{-4}$ (0.1043)&4.185*10$^{-4}$ (0.1239)\\\hline
2&6.352*10$^{-4}$ (0.0998)&5.842*10$^{-4}$ (0.1008)&5.565*10$^{-4}$
(0.1025)\\\hline
\end{tabular}
\end{table}

There is no noticeable difference between the means of the three statistics. We did run multiple 100,000-sample sets, and the orders (with respect to the absolute value) are different for different sets. However, the standard deviations of the three statistics were clearly different. The standard deviations of the sample means were close to 0.1, which is the theoretical value. The standard deviations of the numbers where the data points converged to by the blurring process were closer to those of the sample mean, and were smaller than those by the nonblurring process. This suggests that the blurring process produced more efficient? estimates than the nonblurring process.

\subsection{Robustness to outliers}

n this experiment, each data set has 95 data points sampled from $N(0,1)$ and another 5 data points from $N(5,1)$. We consider $\tau$ to be 0.5, 1, or 2. For each $\tau$ value, we simulated 100,000 data sets.

By Corollary \ref{cor:f>0}, all data points should converge to a single number. However, due to the floating precision, the outliers which are far from most of the data points may converge to different numbers. For both the blurring and
the nonblurring process, we take the number that most of data points converged to as the statistic. The results are presented in Table \ref{tab:noise}. While the sample mean was no longer an unbiased estimator of the true mean when outliers are present, Table \ref{tab:noise} shows that the numbers where most of data points converged to by the blurring and the nonblurring processes were still very close to the true mean of zero. This suggests that both processes remained to produce good estimates for the mean. The standard deviations produced by the blurring process were again smaller than those by the nonblurring process.

\begin{table}[!h]
\caption{The mean and the standard deviation of the converged points
with 5\% outliers }\label{tab:noise} \centering
\begin{tabular}{|l|c|c|c|}\hline
$\tau$ & Sample Mean & Blurring & Nonblurring\\\hline 0.5&0.2495
(0.1003)& -0.0006 (0.1241) & -0.0038 (0.2167)\\\hline 1&0.2495
(0.1000)& -0.0106 (0.1102) &  0.0002 (0.1276)\\\hline 2&0.2503
(0.0998)&0.0928 (0.1046)& 0.0220 (0.1080)
\\\hline
\end{tabular}
\end{table}

\section{Discussion and Conclusion}

In this paper, we first give a rigorous mathematical proof of the convergence of the blurring mean-shift process. Our result is under the condition that $f$ is PDD and $w$ depends only on data points.

We also prove the consistency of the blurring process, which ensures the estimation to converge to the true values of the parameters as the number of data points goes to infinity. Our consistency proof is for the Normal case, in which we could show the
explicit form of the shrinkage rate of the data points. The consistency for more general kernel functions can be proven in similar arguments.

From our simulation studies, both the blurring and the nonblurring processes have good robustness against outliers. The estimations by the blurring process usually yield smaller variances than those by the nonblurring process.

\section*{Acknowledgements}
The author would like to thank Pei Lun Tseng for suggesting a
shorter proof on Lemma 2, and Professor Chii-Ruey Hwang and
Professor Su-Yun Huang for inputs and discussions.

\bibliographystyle{apalike}      
\bibliography{conv}   

\end{document}